\documentclass[letterpaper, 10 pt, conference]{ieeeconf}  

\IEEEoverridecommandlockouts                              

\usepackage{bm}
\usepackage[dvipsnames]{xcolor}
\usepackage{multicol}
\usepackage{hyperref}
\usepackage{amsmath}
\usepackage{graphicx}
\usepackage{amssymb}
\usepackage{float}
\usepackage{mathrsfs}
\usepackage{comment}
\usepackage{mathtools}
\usepackage{yfonts}
\usepackage{caption}
\usepackage{subcaption}
\usepackage{float}
\usepackage{algorithm}
\usepackage{algpseudocode}
\usepackage{xcolor}
\usepackage{commands}

\hypersetup{
  colorlinks=true,
  citecolor=blue,
  linkcolor=red,
  urlcolor=Green
}

\title{\LARGE \bf
Efficient Constrained Multi-Agent Trajectory Optimization using Dynamic Potential Games
}

\author{Maulik Bhatt$^{1}$, Yixuan Jia$^{2}$, and Negar Mehr$^{1}$
\thanks{This work is supported by the National Science Foundation, under grants ECCS-2145134 CAREER Award, CNS-2218759, and CCF-2211542.}
\thanks{$^{1}$Maulik Bhatt and Negar Mehr are with the Department of Aerospace Engineering, University of Illinois Urbana-Champaign, USA
        {\tt\small  mcbhatt2@illinois.edu, negar@illinois.edu }}%
\thanks{$^{2}$Yixuan Jia is with the Department of Electrical and Computer Engineering, University of Illinois Urbana-Champaign, USA
        {\tt\small yixuanj3@illinois.edu}}%
}

\begin{document}

\maketitle
\thispagestyle{empty}
\pagestyle{empty}

\begin{abstract}

Although dynamic games provide a rich paradigm for modeling agents' interactions, solving these games for real-world applications is often challenging. Many real-world interactive settings involve general nonlinear state and input constraints that couple agents' decisions with one another. In this work, we develop an efficient and fast planner for interactive trajectory optimization in constrained setups using a constrained game-theoretical framework. Our key insight is to leverage the special structure of agents' objective and constraint functions that are common in multi-agent interactions for fast and reliable planning. More precisely, we identify the structure of agents' cost and constraint functions under which the resulting dynamic game is an instance of a constrained dynamic potential game. Constrained dynamic potential games are a class of games for which instead of solving a set of \emph{coupled} constrained optimal control problems, a constrained Nash equilibrium, i.e. a Generalized Nash equilibrium, can be found by solving a \emph{single} constrained optimal control problem. This simplifies constrained interactive trajectory optimization significantly. We compare the performance of our method in a navigation setup involving four planar agents and show that our method is on average 20 times faster than the state-of-the-art. We further provide experimental validation of our proposed method in a navigation setup involving two quadrotors carrying a rigid object while avoiding collisions with two humans.    

\end{abstract}

\IEEEpeerreviewmaketitle

\vspace{-0.1cm}
\section{Introduction}
\vspace{-0.1cm}
Many robotic applications such as autonomous driving, crowd-robot navigation, and delivery robots involve multi-agent interactions where a robot must interact with either humans or other robots in the environment. In these applications, each agent may need to satisfy some constraints such as collision avoidance and goal constraints.
Trajectory planning in such settings is a challenging task as agents' decisions are normally coupled through their objectives and constraints. For example, agents' actions may be coupled through collision avoidance constraints that every agent tries to maintain.
Dynamic games provide a strong mathematical foundation for reasoning about such couplings among the agents in sequential decision-making tasks~\cite{bacsar1998dynamic}, where the interaction outcome is captured by the equilibria of the dynamic game. 
However, finding equilibria of dynamic games is challenging as it normally involves solving a set of nonlinear coupled optimal control problems subject to nonlinear state and action constraints. As a result, solving for equilibria of such games is computationally demanding.

In this work, we identify a structure under which constrained multi-agent interactive planning reduces to solving a single constrained optimal control problem. This reduction simplifies the planning problem significantly and results in remarkably faster and more efficient solution methods.
More precisely, we identify the structure of agents' cost and constraint functions under which, the resulting dynamic game is an instance of a constrained dynamic potential game. Potential games are a well-studied class of games for which a Nash equilibrium always exists, and an equilibrium can be found by solving a single optimization problem~\cite{monderer1996potential}. In this work, we show that constrained dynamic potential games provide a simple and efficient method for planning in constrained interactive domains.

We introduce the single constrained optimal control problem whose solutions will be constrained Nash equilibria -- a.k.a Generalized Nash equilibria -- of multi-agent interactions. The resulting constrained optimal control problem can be solved by any generic solver. In this work, we show that one can use off-the-shelf standard solvers such as ALTRO~\cite{howell2019altro} and do-mpc \cite{LUCIA201751}, for efficient multi-agent trajectory planning with nonlinear state and action constraints. 
Our main contributions are the following: 
\begin{itemize}
    \item We identify the structure under which constrained interactive trajectory planning is an instance of a constrained dynamic potential game.
    \item We develop an efficient algorithm for seeking Generalized Nash equilibria in constrained interactive settings.
    \item We showcase the capabilities of our method in a four-agent planner navigation setting and compare the performance of our algorithm with the state-of-the-art. We also validate our method in an experiment setting where two drones carry an object while moving around humans.
\end{itemize}

\vspace{-0.6cm}
\section{Related Work}
\vspace{-0.1cm}
\subsection{Game-Theoretic Planning}
\vspace{-0.1cm}

\begin{figure*}[t]
    \centering
    \vspace{0.25cm}
    \includegraphics[scale=0.63]{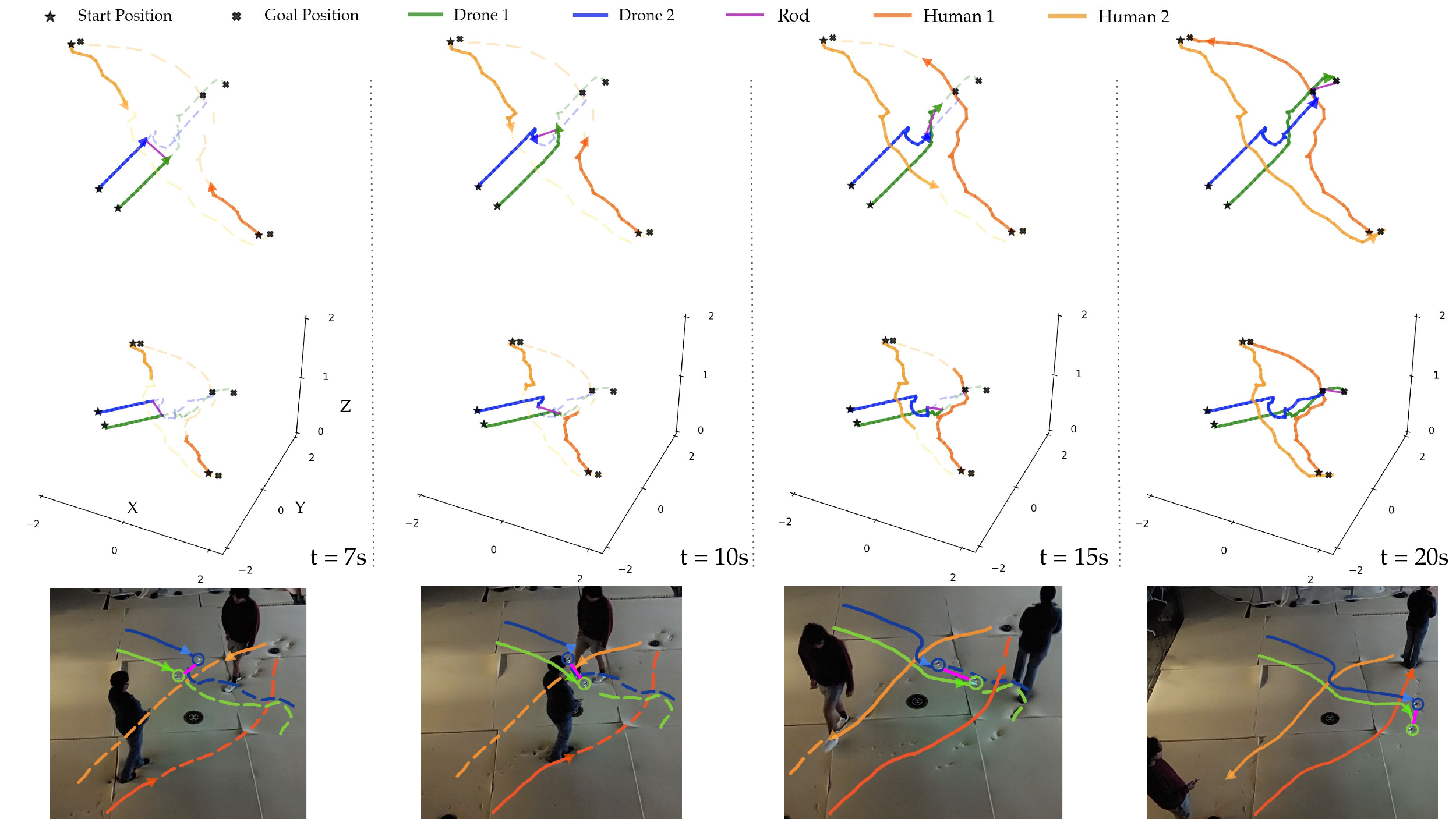}
    \caption{Visualizations and the corresponding snapshots of the hardware experiment. 
    The two quadrotors and two humans need to reach their designated goal positions while avoiding collision with others. The two quadrotors are carrying a rod. In this case, each quadrotor needs to account for the motions of humans as well as the other quadrotor as they share the constraint imposed by the rod.
    As shown in the figure, with our algorithm, the two quadrotors are able to avoid nearby humans by changing their orientations elegantly while carrying the rod. }
    \label{fig:hardware_experiment}
    \vspace{-0.55cm}
\end{figure*}

Various methods have been proposed for tractable game-theoretic planning. Initially, a Stackelberg equilibrium was considered in two-player games for modeling the mutual coupling between two agents ~\cite{sadigh2016planning,liniger2019noncooperative,yoo2012stackelberg}. In the context of autonomous driving, a Markovian Stackelberg strategy was developed using hierarchical planning in~\cite{fisac2019hierarchical}. However, the proposed dynamic programming approach suffers from the curse of dimensionality. Moreover, it was shown that Stackelberg equilibria may not be expressive enough for capturing the mutual couplings among the agents and may fail even in simple settings~\cite{facchinei2010generalized}.


Game-theoretic methods that seek Nash equilibria were later developed to relax the asymmetry of Stackelberg equilibria and enable planning with more than two agents. Iterative best-response methods were developed in~\cite{spica2020real, wang2019game} for autonomous racing. 
However, iterative best-response methods may require a large number of iterations, and; thus, may be computationally prohibitive. Sequential
linear-quadratic methods were applied to two-player zero-sum differential games in~\cite{mukai2000sequential,tanikawa2012local}.
Methods similar to differential dynamic programming were developed for game-theoretic planning. These methods are fast enough to be run in real-time and scale to more than two agents~\cite{fridovich2020efficient,di2018differential}. Similar methods were
proposed in~\cite{wang2020game,mehr2023maximum} for finding equilibria of stochastic dynamic games where planning was considered in the presence of uncertainties. However, these methods cannot handle state and input constraints. In such methods, task constraints can only be included as soft constraints in the objective functions of the agents. 

\vspace{-0.2cm}
\subsection{Generalized Nash Equilibria}
\vspace{-0.1cm}
When agents are subject to state and input constraints that couple their decisions, Generalized Nash equilibria must be sought. For static problems, various methods have been proposed for seeking generalized Nash equilibria ranging from penalty methods to Newton-style methods~\cite{facchinei2009generalized,facchinei2010penalty}. For dynamic games, in~\cite{cleac2019algames}, open-loop Nash equilibria of constrained dynamic games were found using a quasi-Newton root-finding algorithm while enforcing constraints using an augmented Lagrangian
formulation. An algorithm for approximating  Generalized feedback Nash trajectories was presented in~\cite{laine2021computation} using a Newton-style method.

\vspace{-0.2cm}
\subsection{Potential Games}
\vspace{-0.2cm}
Potential games are a well-studied class of games. In~\cite{rosenthal1973class,monderer1996potential}, static potential games with pure Nash equilibria were studied. 
Potential games appear in many engineering applications. Their simplicity and tractability have warranted their implementation across various applications such as resource allocation, cooperative control, power networks, and communication networks~\cite{zazo2014control,zazo2016dynamic, arslan2007autonomous,marden2009cooperative,candogan2010near}. Recently, potential games were used in~\cite{kavuncu2021potential, 10178387, williams2023distributed} for interactive planning. However, no constraints were considered in~\cite{kavuncu2021potential, williams2023distributed}. In this paper, we generalize these works and demonstrate how potential games can facilitate fast and efficient constrained interactive trajectory planning in multi-agent domains.

\vspace{-0.2cm}
\section{Problem Formulation}


Let $\aset := \{1,\ldots,N\} $ denote the set of the agents' indices. Let $\astates^i \subseteq \bR ^{n_i}$ be the set of states for each agent $i \in \aset$, where $n_i$ is the dimension of the state space of agent $i$. The overall state space of the game is denoted as $\astates := \astates^1\times \cdots\times\astates^{N} \subseteq \bR^n $ where $n = \sum_{i\in\aset}n_i$ is the state dimension of the game. We assume that the game is played over $T$ time steps. The size of each time step is denoted by $h$.  
The state-vector of the game at time step $k, 0 \leq k \leq T$ is denoted by $ x_k := \left(x^1_k,\ldots,x^N_k\right) \in \astates$, where $x^i_k \in \astates^i$ denotes the state-vector of agent $i$ at time step $k$. Note that the state vector of the system contains the state of all the agents. For each agent $i \in \aset$, the action set is denoted by $\actionset^i \subseteq \bR^{m_i}$ where $m_i$ is the dimension of the action space of agent $i$.
The set of actions of the game is denoted by $\actionset = \actionset^1\times\cdots\times\actionset^N \subseteq \bR^{m}$ where $m=\sum_{i\in\mathcal{N}}m_i$ is the control dimension of the game. The action or action vector of the game at time step $k$ is denoted by $ u_k := \left(u^1_k,\ldots,u^N_k\right) \in \actionset$ where $u^i_k \in \actionset^i$ is the action vector of agent $i$ at time step $k$. For every agent $i$ in $mathcal{N}$, let $\actionset^{-i} := \Pi_{j\neq i}\actionset^j$ be the set of actions for all agents except agent $i$. We also define $u^{-i}_k := (u^1_k,\ldots,u^{i-1}_k,u^{i+1}_k,\ldots,u^N_k) \in \actionset^{-i}$ as the vector of all agents' actions except agent $i$ at time step $k$. By a slight abuse of notation, we can write $u_k = (u^i_k,u^{-i}_k)$. The discrete-time system dynamics are defined by $f: \astates\times\actionset\times\{0,1,\ldots,T-1\}\rightarrow \astates$:
\vspace{-0.2cm}
\begin{equation}
    x_{k+1} = f(x_k,u_k,k).\label{sys_dynamics}
\end{equation}
\vspace{-0.6cm}

\noindent In order to account for constraints, we assume that there is an inequality constraint function $g:\astates\times\actionset\times\{0,1,\ldots,T-1\}\rightarrow \bR^c$ where $c$ is the number of constraints. The inequality constraints are defined as
\vspace{-0.2cm}
\begin{equation}\label{constraints}
    g(x_k,u_k,k) \leq 0,
\end{equation}
\vspace{-0.6cm}

\noindent where the inequality is interpreted elementwise. For the terminal time-step, the inequality constrained is only a function of the terminal state, $g(x_T,T)\leq0$.
It is to be noted that any equality constraint can also be included as a combination of two inequality constraints. Also, we assume that constraints that capture coupling between agents are shared among the agents. Let $\mathcal{C}_k$ at each time step $k$ be the set of all the feasible states $x_k$ and actions $u_k$. We can define $\mathcal{C}_0 = \astates\times\actionset \cap \{(x_0,u_0): g(x_0,u_0)\leq 0\}$ and $\mathcal{C}_k = \{\{\astates \cap \{x_k \mid x_k = f(x_{k-1},u_{k-1},k-1)\} \}\times\actionset\}\cap \{(x_k,u_k): g(x_k,u_k)\leq 0\}$ for $k\in\{1,\ldots,T-1\}$. Furthermore, at the terminal time step $T$, no action is being taken, and; hence, the constraint function will only be a function of states, and the corresponding constraint set will be $\mathcal{C}_{T} = \astates \cap \{x_{T}: g(x_{T},T) \leq 0 \}$.

We denote the strategy space for agent $i$ as $\Gamma^i$. Let the strategy $\gamma^i \in \Gamma^i$ of agent $i$ be given by $ \gamma^i:\astates\times\{0,1,\ldots,T\}\rightarrow\actionset^i$ which determines the actions of the agent at all time instants. We consider open-loop strategies that are only a function of the system's initial state and time. Therefore, we have the following relation, $\gamma^i(x_0,k) := u^i_k$.

We would like to point out that in this paper, we focus on finding open-loop equilibria of the dynamic game, which implies that at equilibrium, the entire trajectory of each agent is the best response to the trajectories of all the other agents for a given initial state of the system, i.e., the control action of every agent is only a function of time step $k$ and the initial state. We will solve for open-loop Nash equilibria repeatedly in a receding-horizon fashion to adapt to new information that is obtained over time and mimic a feedback policy. We acknowledge that for general dynamic games, due to the difference between the information structure of open-loop and feedback Nash equilibria, the two concepts may result in different behaviors. However, for many trajectory planning purposes, a receding-horizon implementation of open-loop equilibria results in reasonable approximations. 

We denote the combined strategy of all the agents by $\gamma := (\gamma^1,\ldots,\gamma^N) \in \Gamma$ where $\Gamma := \Gamma^1\times\cdots\times\Gamma^N$ is the strategy space of the whole system. We also denote $\gamma = (\gamma^i,\gamma^{-i})$ where $\gamma^{-i} = (\gamma^1,\ldots,\gamma^{i-1},\gamma^{i+1},\ldots,\gamma_N)$ is the tuple containing strategies of all agents except for agent $i$. Note that a given strategy $\gamma$ identifies a unique set of actions for all time instants, $\{u_k\}_{k\in\{0,1,\ldots,T-1\}}$, and; therefore, there is an equivalence between the two $\left(\gamma\equiv\{u_k\}_{k\in\{0,1,\ldots,T-1\}}\right)$. Hence, for simplicity, we use strategies and actions interchangeably from now on. Furthermore, in the case of open-loop strategies, actions at all time steps can be determined by the initial state of the system, and a given strategy will in turn determine the states of the system at each time step. Therefore, the state $x_k$ can be written as a function of strategies, initial states, and the time step, i.e. $x_k \equiv x_k(\gamma,x_0)$. We omit denoting this dependence when it is obvious from the context but denote explicitly when required.

We assume that every agent has a cost function $J^i:\astates\times\Gamma \rightarrow \bR$ that depends on the initial state and the strategies of all agents via
\vspace{-0.23cm}
\begin{equation}
    J^i(x_0,\gamma) = S^i(x^i_{T},T) + \sum_{k=0}^{T-1}L^i(x^i_k,\gamma(x_0,k),k),
\end{equation}
\vspace{-0.37cm}

\noindent where $L^i$ and $S^i$ are the running and terminal costs of agent $i$ respectively. Note that we assume that every agent's cost depends on its own state and the actions of all agents. The agents' couplings are captured through the shared task constraints. The goal of each agent is to choose a strategy $\gamma^i$ that minimizes its cost function while satisfying the constraints. However, since generally, it is not possible for all agents to optimize their costs simultaneously, we model the outcome of the interactions between the agents as equilibria of the underlying dynamic game. 
We represent the constrained dynamic game in a compact form as $\mathcal{G} := \left( \aset, \{\Gamma_i\}_{i\in\aset}, \{J_i\}_{i\in\aset},\{\mathcal{C}_k\}_{k\in\{0,\ldots,T\}},f \right)$. In particular, we model the outcome of interaction among the agents by open-loop generalized Nash equilibria of the underlying dynamic game. The open-loop generalized Nash equilibria of the game are defined as:
\begin{definition}\label{gen_NE}
An open-loop generalized Nash equilibrium of a game $\mathcal{G} := \left( \aset, \{\Gamma_i\}_{i\in\aset}, \{J_i\}_{i\in\aset},\{\mathcal{C}_k\}_{k\in\{0,\ldots,T\}},f \right)$ is a feasible strategy $\gamma^{*} = \left( {\gamma^1}^*,\ldots,{\gamma^N}^* \right)$ for which the following holds for each agent $i \in \aset$:
\vspace{-0.08cm}
\begin{align}\label{eq:Nash-definition}
    & J^i(x_0,\gamma^{*}) \leq   J^i(x_0,{\gamma^{i}},{\gamma^{-i}}^*) 
    \nonumber\\
    & \left(x_k,{\gamma^{i}}(x_0,k),{\gamma^{-i}}^*(x_0,k)\right) \in \mathcal{C}_k, k \in \{0,\ldots,T-1\}, \nonumber \\
    & \; x_{T} \in \mathcal{C}_{T}.
\end{align}
\end{definition}

Intuitively, at equilibrium, no agent can decrease their cost function by unilaterally changing their strategies to any other feasible strategy. In generalized Nash equilibria, the strategies of the agents are further coupled due to the joint constraints of the agents. Note that solving~\eqref{eq:Nash-definition} requires solving a set of $N$ coupled constrained optimal control problems which are normally hard to solve tractably for robotic systems with general nonlinear dynamics, cost, and constraint functions. 


\section{Constrained dynamic potential games}
There is a class of dynamic games called dynamic potential games for which we can associate a single multivariate optimal control problem whose solutions are the Nash Equilibria of the original game. 
Our key insight is that multi-agent constrained trajectory optimization can be cast as a dynamic potential game. We focus our analysis on a generalized version of dynamic potential games where the states and inputs are coupled through task constraints. In the following, we define constrained dynamic potential games from \cite{zazo2016dynamic}:
\begin{definition}
A noncooperative constrained dynamic game $\mathcal{G}$ is a constrained dynamic potential game if there exist functions $P:\astates\times\Gamma\times\{0,1,\ldots,T-1\} \rightarrow \bR$ and $R:\astates\times\{T\} \rightarrow \bR$ such that for every agent $i\in \aset$, any pair of strategies $\gamma^i$ and $\nu^i$, and any strategy vector $\gamma^{-i} \in \Gamma^{-i}$ such that $(x_k,u_k)\in\mathcal{C}_k$ for all $k\in\{0,1,\ldots,T-1\}$ and $x_{T} \in, \mathcal{C}_{T}$, we have the following for every any initial condition $x_0$:
\vspace{-0.2cm}
\begin{align}\label{eq:potential-condition}
    & J^i(x_0,\gamma^i,\gamma^{-i}) - J^i(x_0,\nu^i,\gamma^{-i}) \nonumber\\ &\quad = J(x_0,\gamma^i,\gamma^{-i}) - J(x_0,\nu^i,\gamma^{-i}),
\end{align}
where $J(x_0,\gamma):= R(x_{T},T) + \sum_{k=0}^{T-1}P(x_k,\gamma(x_0,k),k)$.
\end{definition}

Intuitively, condition~\eqref{eq:potential-condition} implies that a game is a potential game if the change in the running and terminal costs of each agent can be captured by the two functions $R$ and $P$ which are common among all agents. In other words, there exist two functions $P$ and $R$ that can capture the rate of change in each agent's cost when the strategies of all the other agents are fixed.

The following equivalent lemmas from \cite{zazo2016dynamic} give conditions under which a game $\mathcal{G}$ is a constrained dynamic potential game. 

\begin{lemma}\label{lemma-1}
A noncooperative constrained dynamic game $\mathcal{G}$ is a constrained dynamic potential game if there exist functions $P:\astates\times\Gamma\times\{0,1,\ldots,T-1\} \rightarrow \bR$ and $R:\astates\times\{T\} \rightarrow \bR$ such that the following holds for every agent $i \in \aset$ and every time step $k = 0,1,\ldots,T-1$:
\begin{align}
\frac{\partial L^i(x^i_k,u_k,k)}{\partial x^i_k} &= \frac{\partial P(x_k,u_k,k)}{\partial x^i_k} \nonumber \\
\frac{\partial L^i(x^i_k,u_k,k)}{\partial u_k^i} &= \frac{\partial P(x_k,u_k,k)}{\partial u_k^i}, \quad \text{and} \nonumber  \\
\frac{\partial S^i(x^i_{T},T)}{\partial x^i_T} &= \frac{\partial R(x_{T},T)}{\partial x^i_T}.
\end{align}
\end{lemma}

\begin{lemma}\label{lemma-2}
A game $\mathcal{G}$ is a constrained dynamic potential game if the running costs and terminal costs function of every agent $i\in\aset$ can be expressed as the sum of a term that is common to all players plus another term that depends neither on its own action nor on its own state-components:
\begin{align}
    L^i(x^i_k,u^i_k,u^{-i}_k,k) & = P(x_k,u^i_k,u^{-i}_k,k) + \Theta^i(x_k^{-i},u_k^{-i},k) \nonumber \\
    S^i(x^i_{T},T) & = R(x_{T},T) + \Phi^i(x^{-i}_{T},T).
\end{align}
\end{lemma}

We now rewrite Theorem 1 from \cite{zazo2016dynamic} to discuss equilibria of dynamic potential games. We make the following assumptions:
\begin{assumption}\label{assum-1}
Agents' running and terminal costs $L^i$ and $S^i$ are continuously differentiable on $\astates\times\Gamma$ and $\astates$.
\end{assumption}
\begin{assumption}\label{assum-3}
The systems dynamics $f$ and constraints $g$ are continuously differentiable in $\astates\times\actionset$ and satisfy some regularity conditions (such as Slater's, the linear independence of gradients, or the Mangasarian-Fromovitz constraint qualification).
\end{assumption}
Under these assumptions, we have the following result:
\begin{theorem}\label{thm-1}
Suppose that $\mathcal{G}$ is a constrained dynamic potential game, and in addition, Assumptions \ref{assum-1}-\ref{assum-3} hold. Then, a solution to the following multivariate optimal control problem:
\vspace{-0.25cm}
\begin{align}\label{mopc}
      \underset{\gamma \in \Gamma}{\text{minimize}} \quad & R(x_{T},T) + \sum_{k=0}^{T-1}P(x_k,\gamma(x_0,k),k) \nonumber\\
     \text{subject to} \quad& x_{k+1} = f(x_k,u_k,k), \; x_0 \; \text{given} \nonumber \\
    & g(x_k,u_k,k) \leq 0,
\end{align} 
\vspace{-0.55cm}

\noindent is a Generalized Nash equilibrium for $\mathcal{G}$.
\end{theorem}
\begin{proof}
Refer to the appendix.
\end{proof}

It should be noted that conditions in Lemma \ref{lemma-1} and Lemma \ref{lemma-2} do not require any conditions on agents' constraints because of our assumption that coupling constraints are shared between agents and agents also have knowledge of the individual constraints of the others.



\vspace{-0.2cm}

\section{Constrained Interactive Trajectory Optimization}


\begin{figure*}[th]
    \centering
    \includegraphics[scale = 0.67]{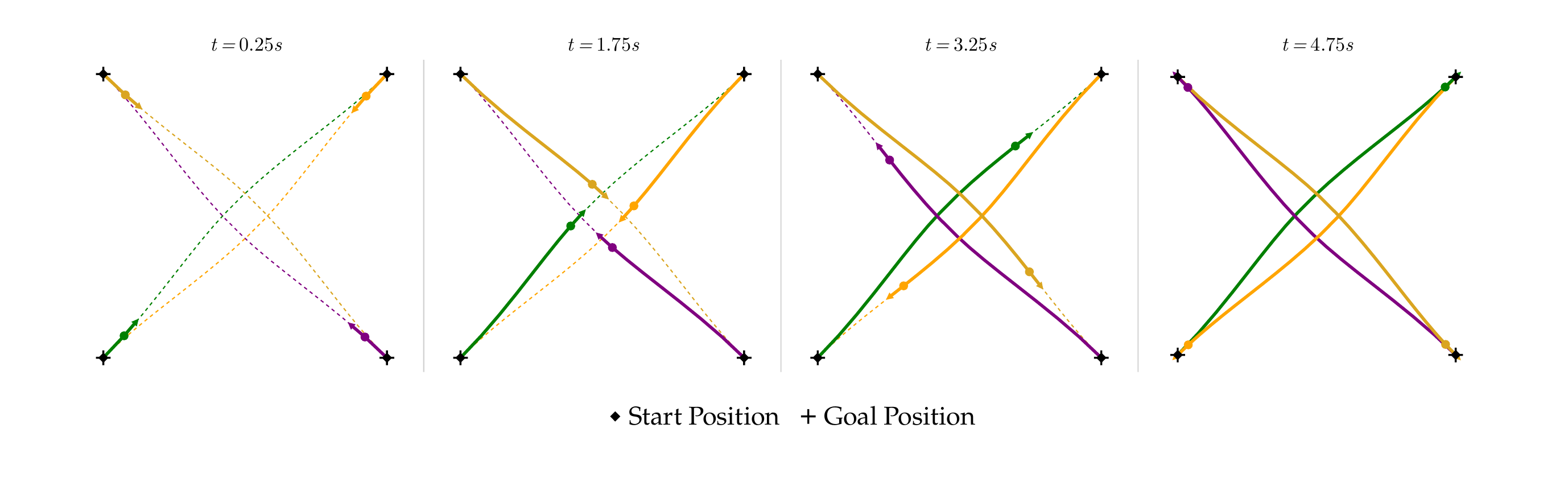}
    \vspace{-1.2cm}
    \caption{The snapshots of the trajectories found by our algorithm when four unicycle agents interact with each other and exchange their positions diagonally over a time interval of 5\emph{s}. Agents move from their start positions to their respective goal positions. In order to avoid collision with other agents, each agent deviates from its nominal trajectory and successfully reaches its goal position.}
    \label{fig:four_agents}
    \vspace{-0.7cm}
\end{figure*}

In this section, we identify the special structure of agents' cost functions under which the constrained equilibria of the underlying dynamic game can be found by solving a single constrained optimal control problem. Agents' decisions are often constrained through their state constraints such as avoiding collision with some stationary or moving obstacle in the environment and avoiding collision with other agents. On the other hand, each agent has an individual goal and task which has to be completed within the given time horizon, i.e. there is also a goal constraint which normally depends only on the states and actions of the agent itself. In such problems, we show that interactions are constrained potential games. We provide the following theorem which characterizes the cost and constraint structures under which the game is a constrained dynamic potential game.

\begin{theorem}\label{thm-2}
If each agent's cost function depends only on its own states and control inputs, i.e. if $J^i(\cdot)$ is of the form,
\vspace{-3mm}
\begin{equation}\label{cost_structure}
    J^i(x_0,\gamma) = S^i(x^i_{T},T) + \sum_{k=0}^{T-1}L^i(x^i_k,u^i_k,k),
\end{equation}

\noindent with system dynamics~\eqref{sys_dynamics} and constraints~\eqref{constraints}, the game $\mathcal{G} = \left( \aset, \{\Gamma_i\}_{i\in\aset}, \{J_i\}_{i\in\aset},\{\mathcal{C}_k\}_{k\in\{0,\ldots,T\}},f \right)$ is a constrained dynamic potential game with potential functions
\vspace{-0.2cm}
\begin{align}\label{potential_fun}
    P(x_k,u_k,k) & = \sum_{i\in\aset}L^i(x_k^i,u^i_k,k), k \in \{0,1,\ldots,T-1\}\nonumber \\
    R(x_{T},T) &= \sum_{i\in\aset}S^i(x^i_{T},T),
\end{align}
\vspace{-0.3cm}

\noindent and a generalized Nash equilibrium of the game can be computed by solving the following:
\begin{align}\label{mopc-thm}
      \underset{\gamma \in \Gamma}{\text{minimize}} \quad & R(x_{T},T) + \sum_{k=0}^{T-1}P(x_k,\gamma(x_0,k),k) \nonumber\\
     \text{subject to} \quad& x_{k+1} = f(x_k,u_k,k), \; x_0 \; \text{given} \nonumber \\
    & g(x_k,u_k,k) \leq 0.
\end{align}
\end{theorem}
\begin{proof}
It is straightforward to see that the running cost of each agent $i\in\aset$ can be written as
\begin{align}
    L^i(x_k^i,u^i_k,k) & = \sum_{i\in \aset}L^i(x_k^i,u^i_k,k) - \sum_{j\neq i, j\in \aset}L^j(x_k^j,u^j_k,k) \nonumber \\
    &= P(x_k,u_k,k) + \Theta^i(x^{-i}_k,u^{-i}_k,k),
\end{align}
where $P(x_k,u_k,k)$ is the term common to all the agents and $\Theta(x^{-i}_k,u^{-i}_k,k) = \sum_{j\neq i, j\in \aset}L^j(x_k^j,u^j_k,k)$ is the term that does not depend on the action, nor states of agent $i$. Similarly, for the terminal costs, we can write
\begin{align}
    S^i(x_{T}^i,T) & = \sum_{i\in \aset}S^i(x_{T}^i,T) - \sum_{j\neq i, j\in \aset}S^j(x_{T}^j,T) \nonumber \\
    &= R(x_{T},T) + \Phi^i(x_{T}^{-i},T),
\end{align}
where $R(x_{T},T)$ is the term common to all the agents and $\Phi^i(x_{T}^{-i},T) = \sum_{j\neq i, j\in \aset}S^j(x_{T}^j,T)$ is the term that does not depend on the state components of agent $i$. Therefore, using Lemma~\ref{lemma-2}, it is easy to see that the game $\mathcal{G}$ is a constrained dynamic potential game with the potential function \eqref{potential_fun}. Furthermore, assuming that Assumptions~\ref{assum-1}-\ref{assum-3} hold, applying Theorem~\ref{thm-1}, it can be concluded that a generalized Nash equilibrium of the game $\mathcal{G}$ can be computed by solving constrained optimal control problem \eqref{mopc-thm}.
\end{proof}

Note that~\eqref{cost_structure} captures all scenarios where agents have individual costs such as tracking costs, goal-reaching costs, and energy consumption costs subject to both joint and individual constraints. For example, an agent may have individual constraints such as bounds on control inputs while sharing a set of coupled constraints with all other agents such as pairwise collision avoidance with other agents. We believe that such a structure captures a wide range of interactions.


If a game $\mathcal{G} = \left( \aset, \{\Gamma_i\}_{i\in\aset}, \{J_i\}_{i\in\aset},\{\mathcal{C}_k\}_{k\in\{0,\ldots,T\}},f \right)$ satisfies~\eqref{cost_structure}, then using Theorem~\ref{thm-2}, to find generalized Nash equilibria, it suffices to solve~\eqref{mopc-thm}, which is a single-agent constrained optimal control problem. In general, any trajectory planner capable of maintaining constraints can be used for solving~\eqref{mopc-thm}. In this paper, we use the Augmented Lagrangian Trajectory Optimization (ALTRO) algorithm~\cite{howell2019altro} for solving~\eqref{mopc-thm}. ALTRO combines iterative-LQR (iLQR)~\cite{li2004iterative} with an augmented Lagrangian method to handle general state and input constraints and an active-set projection method for final “solution polishing” to achieve fast and robust convergence.

ALTRO computes an augmented Lagrangian $\mathcal{L}(X,\gamma,\lambda,\mu)$ that consists of the Lagrangian of the optimal control problem with extra penalties added for constraint violation. Let $X$ be the concatenated vector of $\{x_k\}_{k\in\{0,1,\ldots,N\}}$, $\lambda$ be the vector containing the Lagrange multipliers and $\mu$ be the penalties for constraints. ALTRO minimizes the augmented Lagrangian using iLQR. Then, the outer loop updates the values of Lagrange multipliers and penalty variables. This is repeated until the maximum constraint violation goes below the tolerance values. Then, ALTRO uses an active set projection method to project previously obtained coarse solutions onto the manifold associated with the active constraints. A brief description of ALTRO is presented in Algorithm~\ref{alg:altro}.

\begin{algorithm}
\caption{ALTRO}\label{alg:altro}
\begin{algorithmic}
\Procedure{}{}
    \State Initialize $x_0$, $\gamma$, tolerances
    \State Compute $X := \{x_k\}_{k\in\{0,\ldots,T\}}$ from $x_0$ and $\gamma$
    \State Initialize $\lambda,\mu,\phi$
    \While{$\max$(\text{constraint violation}) $>$ tolerance}
    \State Minimize the augmented Lagrangian $\mathcal{L}(X,\gamma,\lambda,\mu)$ 
    \State use iLQR and compute updated $X,\gamma$.
    \State Update the $\lambda$ and $\mu$.
    \EndWhile
    \State Use active-set projection to obtain updated $(X,\gamma)$
    
    \State \Return $(X,\gamma)$
\EndProcedure
\end{algorithmic}
\end{algorithm}

\vspace{-0.3cm}


\section{Numerical Simulation}

In this section, we consider a planar navigation setting and compare the performance of our method with the state-of-the-art game solvers. We consider four agents sitting on the corners of a square of length 3m such that their goal positions are the opposite diagonal ends of the square (see Fig.~\ref{fig:four_agents}). Each agent wants to reach its goal position while avoiding collisions with others. We consider the following discrete-time unicycle dynamics to model each agent $i$:
\begin{align}\label{unicycle}
    p^i_{k+1} & = p^i_k + h\cdot v^i_k\cos{\theta^i_k}, \;
    q^i_{k+1} = q^i_k + h\cdot v^i_k\sin{\theta^i_k} \nonumber \\
    \theta^{i}_{k+1} & = \theta^i_k + h\cdot\omega_i,
\end{align}
where $p^i_k$ and $q^i_k$ are $x$ and $y$ coordinates of the positions in the 2D plane, $\theta^i_k$ is the heading angle from positive x-axis, $v^i_k$ is the forward velocity, and $\omega^i_k$ is the angular velocity of the agent $i$ at time step $k$. For each agent $i$, we let the state vector $x^i_k$ be  $[p^i_k,q^i_k,\theta^i_k]$, and the action vector $u^i_k$ be $[v^i_k,\omega^i_k]$.
Suppose the desired position of each agent $i$ is $x^i_f$. We consider the cost function of each agent $i$ to be \vspace{-0.17cm}
\begin{align}
    & J^i(x_0,\gamma) =  \frac{1}{2}\left(x^i_{T} - x_f^i\right)^\top Q_f^i\left(x^i_{T} - x_f^i\right) \nonumber\\ & + \sum_{k=0}^{T - 1}\frac{1}{2}\left(x^i_{k} - x^i_f\right)^\top Q^i\left(x^i_{k} - x^i_f\right) + \frac{1}{2}{u^i}_{k}^\top C^iu^i_{k},
\end{align}
where $Q^i$ and $C^i$ are diagonal matrices that determine the cost of deviation from the reference position and control costs respectively. The matrix $Q_f^i$ determines the cost of deviating from the desired position at the final time instance. To enforce collision avoidance among the agents, we consider a minimum separation distance between the agents. We require the distance between any pair of agents to be greater than or equal to some threshold distance $d_{\text{collision}}$. Therefore, we impose collision avoidance constraints of the following form between any pair of agents
\vspace{-0.15cm}
\begin{align}\label{collision_constraint}
     g_{ij}(x^i_k,x^j_k,k):= -d(x^i_k,x^j_k) + d_{\text{collision}} \leq 0,
\end{align}

\noindent where $d(x^i_k,y^i_k)$ is the distance between agents $i$ and $j$ at time step $k$. We consider the value of $d_{\text{collision}}$ to be 0.3\emph{m}. 
Furthermore, we consider additional constraints for bounding the actions of every each agent:
\begin{align}\label{bound_constraint}
    g_i(u^i_k,k) := |u^i_k| - u_{\text{bound}} \leq 0, 
\end{align}
\noindent where $u_{\text{bound}}$ is the bound on actions. We consider the bound $u_{\text{bound}}$ to be 3 m/s for the linear velocity and 3 rad/s for the angular velocity. All these constraints can be concatenated to obtain the constraint vector $g(x_k,u_k,k)$. We now apply the results from Theorem \ref{thm-2} to obtain optimal trajectories.
\begin{figure}
    \centering
    \includegraphics[scale=0.4]{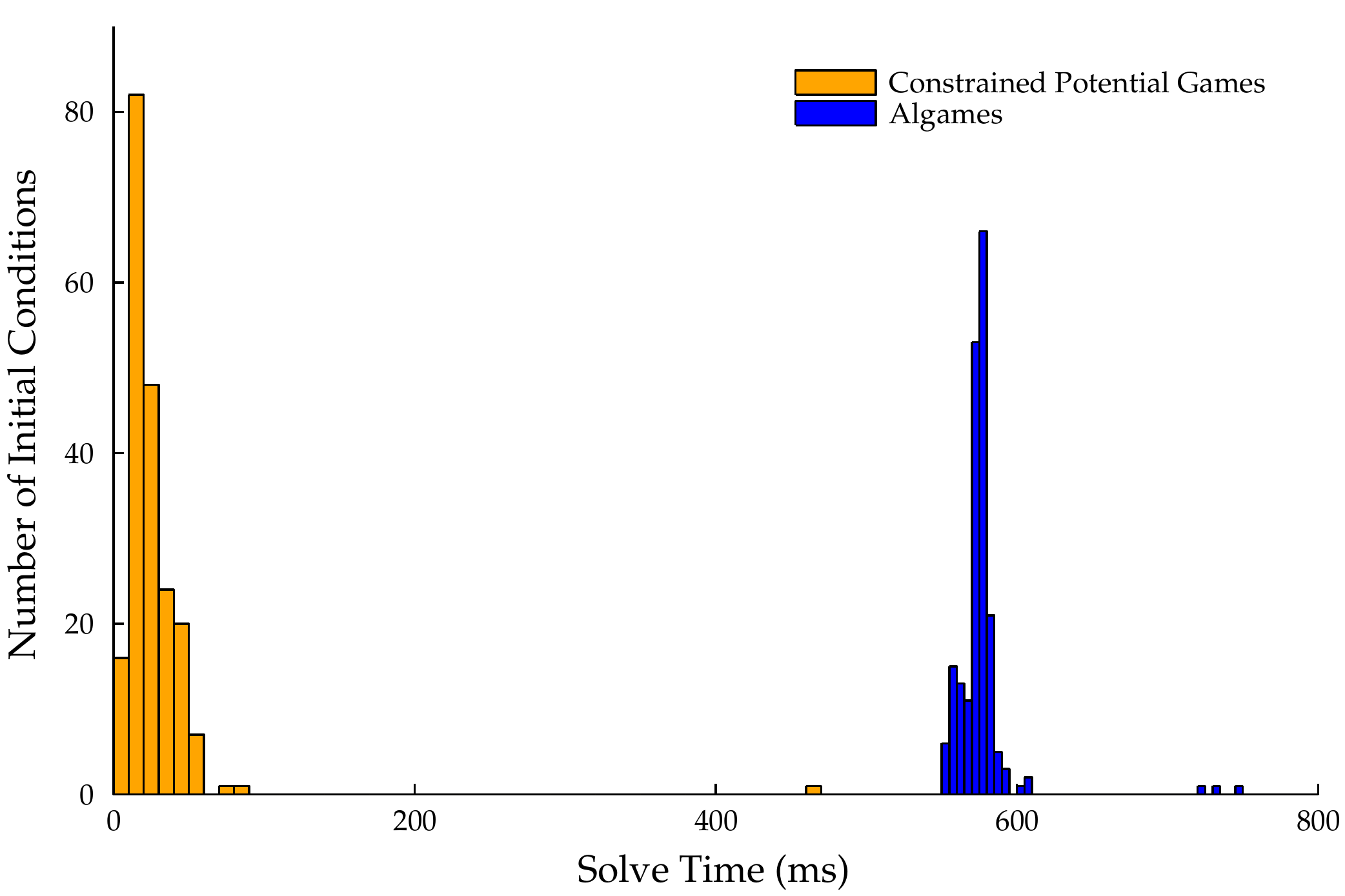}
    \caption{Histograms for the comparison of the solve time of Constrained Potential Games with Algames. Constrained Potential Games has an average solve time of 26.15 \emph{ms} in contrast to Algames which has an average solve time of 577.35 \emph{ms}. Our approach is approximately 20 times faster than Algames.}
    \label{fig:comparison}
    \vspace{-0.7cm}
\end{figure} 
The simulation of the resulting trajectories is demonstrated in~Fig.~\ref{fig:four_agents} where the agents manage to successfully avoid collisions with one another while reaching their designated goal locations. They exhibit intuitive trajectories where they yield and coordinate their motions to avoid collisions. Similar trajectories can be generated for various initial conditions. We further compare the solve time of our algorithm with the state-of-the-art constrained game solver Algames~\cite{cleac2019algames}. We generate 200 random sets of initial conditions for the four agents, and for each set of initial conditions, we run both our algorithm and Algames. We measure the solve time of both methods. The histograms of the solve time for the two methods are plotted in Fig.~\ref{fig:comparison}. The solve time of our method is on average $26.15 \pm 34.30$ \emph{ms} while the average solve time of Algames was $577.35 \pm 27.70$ \emph{ms}. As Fig.~\ref{fig:comparison} shows, our method is significantly faster than Algames, and for most of the initial conditions provides a solution in 20 \emph{ms} while for Algames, the solve time is always greater than 500 \emph{ms}. Our method is on average 20 times faster than Algames. 
\section{Experiments}
 \vspace{0cm}

To demonstrate the performance of our proposed approach, we considered a hardware experiment where two Crazyflie quadrotors are navigating in a room shared with two humans. To enforce complicated task constraints, we require the quadrotors to transport a rigid rod with length $0.5$\emph{m}. This implies that the quadrotors are required to maintain a fixed given distance from each other at all time steps, i.e. they must satisfy an equality constraint at all times. We further assume that agents are subject to collision avoidance constraints. This implies that quadrotors need to satisfy both equality and inequality constraints at each time step. The humans and quadrotors have designated initial and goal positions. The humans are asked to walk to their destinations, and the quadrotors need to carry the rod to the designated location while avoiding collisions with the humans.
The quadrotors use our trajectory optimization algorithm to move from their initial positions to goal positions. 
The state vector of each quadrotor $i \in \{1, 2\}$, at each time step $k$, is:
\vspace{-0.2cm}
\begin{align*}
    x^i_k &= [p^i_{x, k}, ~p^i_{y, k}, ~p^i_{z, k}, ~\phi^i_{k}, ~\theta^i_k, ~\psi^i_k],
\end{align*}
\vspace{-0.5cm}

\noindent which consists of its position ($[p^i_{x, k}, ~p^i_{y, k}, ~p^i_{z, k}]$) and orientation ($[\phi^i_{k}, ~\theta^i_k, ~\psi^i_k]$). We model each quadrotor as a 6 DOF integrator, i.e., each quadrotor dynamics are given by $\dot{x}^i_k = u^i_k$, where $u^i_k = [u^i_{1, k}, ~u^i_{2, k}, ~u^i_{3, k}, ~u^i_{4, k}, ~u^i_{5, k}, ~u^i_{6, k}]$ is the control input. The first three entries of $u^i_k$ correspond to the linear velocities, and the last three entries correspond to angular velocities. This design choice is made based on the fact that we send waypoints as commands to Crazyflies through the Crazyswarm package \cite{crazyswarm}, and they track those waypoints with an in-place low-level controller.
Similar to~\cite{fridovich2020efficient}, we model the humans via unicycle dynamics, except for the fact that there is a constant height $r^i$ associated with each human $i, i \in \{1, 2\}$. We consider the human state vector to be
\vspace{-0.4cm}
\begin{align*}
x^i_k = [p^i_{x, k}, ~p^i_{y, k}, ~r^i, ~\theta^i_{k}] 
\end{align*}
\vspace{-0.6cm}

\noindent where $p^i_{x, k}, ~p^i_{y, k}$ denote the $x, y$ coordinates of human $i$ and $\theta^i_{k}$ denote the orientation of human $i$ (yaw), all at time step $k$.
We assume the $z$-coordinate of the human is located midway from the ground, i.e. $r^i \equiv 1$. Our assumption is that humans are rational and while planning their motion in the environment, they also yield to each other and drones by satisfying constraints. We consider the collision avoidance constraint between the two humans to be the same as~\eqref{collision_constraint}. For the collision avoidance constraint between each quadrotor and each human, we assume that each human occupies a cylinder of height 2m with radius $\sqrt{0.4}\, m$, centered at $(p^i_{x, k}, ~p^i_{y, k}, ~r^i)$. 
The maximum linear velocity for the quadrotors is $1.2$m/s, and we assume that humans have a maximum linear velocity of $1.5$m/s.\\
We run the algorithm in a receding horizon fashion.
We set the time horizon to be 0.5s with a step size of 0.1s. 
We solve the underlying optimization problem with do-mpc
\cite{LUCIA201751}, which models the problem symbolically with CasADi \cite{Andersson2019} and solves them with IPOPT \cite{Wachter:2006wt}.
\\ The resulting trajectories of our experiment are plotted in Fig \ref{fig:hardware_experiment}. 
As shown in the figure, when the two quadrotors are relatively far from the two humans, they move straight toward their designated positions. When there are humans nearby, the quadrotors are able to simultaneously change their orientations and the orientation of the rod to avoid collisions with humans while carrying the rod and maintaining the equality constraints.

\section{Conclusion and Future Work}

In this work, we developed an efficient and reliable algorithm for planning constrained multi-agent interactive trajectories. We identified the cost structures under which game-theoretic interactions become an instance of a constrained dynamic potential game which are games in which a generalized Nash equilibrium can be found by solving a single constrained optimal control problem. We showed that this simplification provides a simple and efficient method for interactive trajectory planning. We demonstrated that in a navigation setup involving four planar agents, our method is on average 20 times faster than the state-of-the-art. We further verified our method in an experiment involving two quadrotors carrying a rigid object around two humans. 

In this work, we considered open-loop Nash equilibria and ran open-loop equilibrium strategies in a receding horizon fashion to approximate feedback policies. For our future work, we aim to examine generalized feedback Nash equilibria. We would like to investigate whether dynamic potential games can be used for efficiently generating generalized feedback Nash strategies.

\vspace{-0.4cm}
\appendix
\vspace{-1mm}\begin{proof}
This theorem was proven in \cite{zazo2016dynamic} for the infinite horizon case when $T = \infty$. Here, we state the proof for the finite horizon case. The Lagrangian for each agent $i \in \aset$ can be written as 
\begin{align}
    & \mathcal{L}(x_k,u_k,\lambda^i_k,\mu^i_k) = S^i(x^i_{T},T) + {\mu^{i^\top}_{T}} g(x_{T},T)\nonumber \\
    & \qquad + \sum_{k=0}^{T-1}\bigg( L^i(x^i_k,u_k,k) + {\lambda^{i}_k}^\top \left( f(x_k,u_k,k) - x_{k+1} \right) \bigg. \nonumber \\
    & \qquad + \bigg. {\mu^{i^\top}_t} g(x_k,u_k,k)\bigg),\end{align}
\vspace{-0.5cm}

\noindent where $\lambda^i_k := (\lambda^{ij}_k)_{j=0}^{n}$ is the corresponding vector of Lagrange multipliers with $n$ being the state dimension of the game, and $\mu_k^i:=(\mu^{il}_k)_{l=0}^c$ with $c$ being number of constraints. In order to obtain the generalized Nash equilibria of the game, the KKT conditions can be written as the following for every agent $i\in\aset$:
\begin{flalign}
    & \frac{\partial L^i(x^i_k,u_k,k)}{\partial x^i_k} + {\lambda^i_k}^\top\frac{\partial f(x_k,u_k,k)}{\partial x^i_k} && \nonumber \\
    & \qquad + {\mu^i_k}^\top\frac{\partial g(x_k,u_k,k)}{\partial x^i_k} - \lambda^{i}_{k-1} = 0, &&
\end{flalign}
\vspace{-0.3cm}
\begin{flalign}
    & \frac{\partial L^i(x^i_k,u_k,k)}{\partial u^{i}_k} + {\lambda^i_k}^\top\frac{\partial f(x_k,u_k,k)}{\partial u^i_k} && \nonumber \\
    & \qquad + {\mu^i_k}^\top\frac{\partial g(x_k,u_k,k)}{\partial u^i_k} = 0, &&
\end{flalign}
\vspace{-0.3cm}
\begin{flalign}
    & x_{k+1} = f(x_k,u_k,k), \; g(x_k,u_k,k) \leq 0, \; \forall \; 0\leq k \leq T-1, &&
    \nonumber\\
    & \frac{\partial S^i(x^i_{T},T)}{\partial x^i_{T}} + {\mu^i_{T}}^\top\frac{\partial g(x_{T},T)}{\partial x^i_{T}} = 0, && \\
    & g(x_{T},T) \leq 0, && \\
    & \mu^i_k \leq 0, \;\; {\mu^i_k}^\top g(x_k,u_k,k) = 0, \; \forall \; 0\leq k \leq T. &&
\end{flalign}
\vspace{-0.6cm}

\noindent The Lagrangian for the multivariate constrained optimal control problem~\eqref{mopc} is
\vspace{-0.2cm}
\begin{align}
    & \mathcal{L}^P(x_k,u_k,\xi_k,\delta_k) = R(x_{T},T) + {\delta_{T}}^\top g(x_{T},T)\nonumber \\
    & \qquad + \sum_{k=0}^{T-1}\bigg( P(x_k,u_k,k) + {\xi_k}^\top \left( f(x_k,u_k,k) - x_{k+1} \right) \bigg. \nonumber \\
    & \qquad + \bigg. {\delta_t}^\top g(x_k,u_k,k)\bigg),
\end{align}
\vspace{-0.6cm}

\noindent where $\xi_k := (\xi^{j}_k)_{j=0}^{n}$ and $\delta_k:=(\delta^{l}_k)_{l=0}^c$ are the corresponding vectors of multipliers. Therefore, the KKT conditions of~\eqref{mopc} are
\vspace{-0.4cm}

\begin{flalign}
    & \frac{\partial P(x_k,u_k,k)}{\partial x^i_k} + {\xi_k}^\top\frac{\partial f(x_k,u_k,k)}{\partial x^i_k} && \nonumber \\
    & \qquad + {\delta_k}^\top\frac{\partial g(x_k,u_k,k)}{\partial x^i_k} - \xi_{k-1} = 0, &&
\end{flalign}
\vspace{-0.3cm}
\begin{flalign}
    & \frac{\partial P(x_k,u_k,k)}{\partial u^{i}_k} + {\xi_k}^\top\frac{\partial f(x_k,u_k,k)}{\partial u^i_k} && \nonumber \\
    & \qquad + {\delta_k}^\top\frac{\partial g(x_k,u_k,k)}{\partial u^i_k} = 0, &&
\end{flalign}
\vspace{-0.4cm}
 \begin{flalign}
    & x_{k+1} = f(x_k,u_k,k), \; g(x_k,u_k,k) \leq 0, \; \forall \;0\leq k \leq T-1 && \nonumber\\
    & \frac{\partial R(x_{T},T)}{\partial x^i_{T}} + {\delta_{T}}^\top\frac{\partial g(x_{T},T)}{\partial x^i_{T}} = 0, && \\
    & g(x_{T},T) \leq 0, && \\
    & \delta_k \leq 0, \;\; {\delta_k}^\top g(x_k,u_k,k) = 0, \forall \; 0\leq k \leq T. &&
\end{flalign}
Thus, in order for multivariate optimal control problem \eqref{mopc} and game $\mathcal{G} = \left( \aset, \{\Gamma_i\}_{i\in\aset}, \{J_i\}_{i\in\aset},\{\mathcal{C}_k\}_{k\in\{0,\ldots,T\}},f \right)$ to have similar optimality conditions, the following must hold:
\begin{align}
\frac{\partial L^i(x^i_k,u_k,k)}{\partial x^i_k} &= \frac{\partial P(x_k,u_k,k)}{\partial x^i_k}, \label{eq:append-1} \\
\frac{\partial L^i(x^i_k,u_k,k)}{\partial u_k^i} &= \frac{\partial P(x_k,u_k,k)}{\partial u_k^i}, \\
\forall i \in \aset, \; k & = 0,1,\ldots,T-1; \; \text{and} \nonumber \\
\frac{\partial S^i(x^i_{T},T)}{\partial x^i_T} &= \frac{\partial R(x_{T},T)}{\partial x^i_T}, \label{eq:append-2} \\
\lambda^i_k = \xi_k, \quad & \mu^i_k = \delta_k, \; \forall i \in \aset.
\end{align}
When conditions \eqref{eq:append-1}-\eqref{eq:append-2} are satisfied, Lemma~\ref{lemma-1} states that $\mathcal{G}$ is a constrained dynamic potential game. Furthermore, since Lemma~\ref{lemma-1} is also a necessary condition for constrained potential dynamic games, we have proven that KKT conditions for solving constrained Nash equilibria are equivalent for original game $\mathcal{G}$ and \eqref{mopc} when a game is a constrained dynamic potential game. Therefore, a solution to the \eqref{mopc} will be a generalized Nash equilibrium of $\mathcal{G}$ for a constrained potential game. Readers are referred to proof of Theorem 1 in \cite{zazo2016dynamic} for further details.
\end{proof}

\vspace{-0.15cm}

\bibliographystyle{IEEEtran}
\bibliography{bibliography}

\end{document}